\documentclass{article}
\usepackage{arxiv}
\usepackage[round, compress]{natbib}
\usepackage{amsmath,amssymb,amsthm}
\usepackage{microtype}
\usepackage{graphicx}
\usepackage{subfigure}
\usepackage{format}
\usepackage{booktabs}
\usepackage{stmaryrd}

\newcommand{\all}{\cX}
\newcommand{\C}{\mat C}
\newcommand{\gps}{\cG_{\text{fair}}}
\newcommand{\wh}{\widehat}
\newcommand{\hmc}{\wh{\mat C}}

\newcommand{\confsh}[1]{\{\mat{C}^g[#1]\}_{g\in\gps}}

\newcommand{\hconfsh}[1]{\{\mat{\wh C}^g[#1]\}_{g\in\gps}}

\newcommand{\honfsh}[1]{\{\mat{\wh C}^g[#1]\}_{g}}

\newcommand{\name}{{\tt GroupFair}}
\def\to{{\,\rightarrow\,}}

\mathchardef\mhyphen="2D

\newcommand{\indicator}[1]{ \left\llbracket {#1} \right\rrbracket }
\newcommand{\vertiii}[1]{{\left\vert\kern-0.25ex\left\vert\kern-0.25ex\left\vert #1
		\right\vert\kern-0.25ex\right\vert\kern-0.25ex\right\vert}}

\def\sut{\text{s.t.}}

\newcommand{\vect}[1]{{\boldsymbol{#1}}}

\def\beeta{\vect{\eta}}
\def\blambda{\vect{\lambda}}

\def\ba{{\mathbf{a}}}

\def\bff{{\mathbf{f}}}

\def\bh{{\mathbf{h}}}

\def\bp{{\mathbf{p}}}

\def\bs{{\mathbf{s}}}

\def\bu{{\mathbf{u}}}
\def\bv{{\mathbf{v}}}
\def\bw{{\mathbf{w}}}
\def\bx{{\mathbf{x}}}
\def\by{{\mathbf{y}}}
\def\bz{{\mathbf{z}}}
\def\bA{{\mathbf{A}}}
\def\bB{{\mathbf{B}}}
\def\bC{{\mathbf{C}}}
\def\bD{{\mathbf{D}}}

\def\bI{{\mathbf{I}}}

\def\bL{{\mathbf{L}}}

\def\bR{{\mathbf{R}}}

\def\bU{{\mathbf{U}}}
\def\bV{{\mathbf{V}}}
\def\bW{{\mathbf{W}}}

\def\bbE{{\mathbb{E}}}

\def\bbP{{\mathbb{P}}}

\def\bbR{{\mathbb{R}}}

\def\cA{\mathcal{A}}

\def\cC{\mathcal{C}}

\def\cE{\mathcal{E}}
\def\cF{\mathcal{F}}
\def\cG{\mathcal{G}}
\def\cH{\mathcal{H}}

\def\cL{\mathcal{L}}

\def\cO{\mathcal{O}}

\def\cU{\mathcal{U}}
\def\cV{\mathcal{V}}

\def\cX{\mathcal{X}}

\def\cZ{\mathcal{Z}}
\newcommand{\mat}[1]{\mathbf{{#1}}} % matrix notation
\newcommand{\conf}{C} % confusion matrix
\newcommand{\sConf}{\widehat{\conf}} % sample confusion matrix
\makeatletter
\newcommand*\bigcdot{\mathpalette\bigcdot@{1}}
\newcommand*\bigcdot@[2]{\mathbin{\vcenter{\hbox{\scalebox{#2}{$\m@th#1\bullet$}}}}}
\makeatother

\def\argmax{\operatorname{argmax}}
\def\argmin{\operatorname{argmin}}

\newcommand{\balgorithm}  {\begin{algorithm}}
	\newcommand{\ealgorithm}  {\end{algorithm}}
\newcommand{\balgorithmic}{\begin{algorithmic}}
	\newcommand{\ealgorithmic}{\end{algorithmic}}
\newcommand{\bassumption} {\begin{assumption}}
	\newcommand{\eassumption} {\end{assumption}}
\newcommand{\bcorollary}  {\begin{corollary}}
	\newcommand{\ecorollary}  {\end{corollary}}
\newcommand{\bdefinition} {\begin{definition}}
	\newcommand{\edefinition} {\end{definition}}
\newcommand{\bexample}    {\begin{example}}
	\newcommand{\eexample}    {\end{example}}
\newcommand{\bprop}    {\begin{prop}}
	\newcommand{\eprop}    {\end{prop}}
\newcommand{\blemma}      {\begin{lemma}}
	\newcommand{\elemma}      {\end{lemma}}
\newcommand{\bproblem}    {\begin{problem}}
	\newcommand{\eproblem}    {\end{problem}}
\newcommand{\bproof}      {\begin{proof}}
	\newcommand{\eproof}      {\end{proof}}
\newcommand{\bremark}     {\begin{remark}}
	\newcommand{\eremark}     {\end{remark}}
\newcommand{\btheorem}    {\begin{theorem}}
	\newcommand{\etheorem}    {\end{theorem}}

\usepackage{algorithm2e}
\newlength\mylen
\newcommand\myinput[1]{%
  \settowidth\mylen{\KwIn{}}%
  \setlength\hangindent{\mylen}%
  \hspace*{\mylen}#1\\}
\DeclareMathOperator{\mino}{MinOracle}
\DeclareMathOperator{\wmed}{wmedian}
\DeclareMathOperator{\update}{Update}
\DeclareMathOperator{\plugin}{plugin}
\newtheorem{innercustomthm}{Theorem}
\newenvironment{customthm}[1]
  {\renewcommand\theinnercustomthm{#1}\innercustomthm}
  {\endinnercustomthm}
  \newtheorem{innercustomprop}{Proposition}
\newenvironment{customprop}[1]
  {\renewcommand\theinnercustomprop{#1}\innercustomprop}
  {\endinnercustomprop}

\usepackage[utf8]{inputenc} % allow utf-8 input
\usepackage[T1]{fontenc}    % use 8-bit T1 fonts
\usepackage{hyperref}       % hyperlinks
\usepackage{url}            % simple URL typesetting
\usepackage{booktabs}       % professional-quality tables
\usepackage{amsfonts}       % blackboard math symbols
\usepackage{nicefrac}       % compact symbols for 1/2, etc.
\usepackage{microtype}      % microtypography
\makeatletter
\renewcommand{\@algocf@capt@plain}{above}% formerly {bottom}

\makeatother
\title{Fairness with Overlapping Groups}
\author{%
  Forest Yang\thanks{Work completed while an intern at Google Research Accra.} \\
  UC Berkeley\\
  \And
  Moustapha Cisse \\
  Google Research Accra \\
  \AND
  Sanmi Koyejo\\
  Google Research Accra \& Illinois\\
}

\begin{document}
\date{}
\maketitle

\begin{abstract}
In algorithmically fair prediction problems, a standard goal is to ensure the equality of fairness metrics across multiple overlapping groups simultaneously. We reconsider this standard fair classification problem using a probabilistic population analysis, which, in turn, reveals the Bayes-optimal classifier. Our approach unifies a variety of existing group-fair classification methods and enables extensions to a wide range of non-decomposable multiclass performance metrics and fairness measures. The Bayes-optimal classifier further inspires consistent procedures for algorithmically fair classification with overlapping groups. On a variety of real datasets, the proposed approach outperforms baselines in terms of its fairness-performance tradeoff.
\end{abstract}

\section{Introduction}
\label{section:intro}

Machine learning inform an increasingly large number of critical decisions in diverse settings. They assist medical diagnosis~\citep{mckinney2020international}, guide policing~\citep{meijer2019predictive}, and power credit scoring systems~\citep{tsai2008using}. While they have demonstrated their value in many sectors, they are prone to unwanted biases, leading to discrimination against protected subgroups within the population. For example, recent studies have revealed biases in predictive policing and criminal sentencing systems~\citep{meijer2019predictive,Chouldechova17}. The blossoming body of research in algorithmic fairness aims to study and address this issue by introducing novel algorithms guaranteeing a certain level of non-discrimination in the predictions. Each such algorithm relies on a specific definition of fairness, which falls into one of two categories: Individual fairness~\citep{Dwork2012,Zemel13} or group fairness~\citep{Calders10, Kamishima11, Hardt16}. The vast majority of the algorithmic group fairness literature has focused on the simplest case where there are only two groups. In this paper, we consider the more nuanced case of group fairness with respect to multiple groups. 

The simplest setting is the {\em independent} case, with only one sensitive attribute which can take multiple values, e.g., race only. The presence of multiple sensitive attributes (e.g., race {\em and} gender simultaneously) leads to non-equivalent definitions of group fairness. On the one hand, fairness can be considered independently per sensitive attribute, leading to overlapping subgroups. For example, consider a model restricted to demographic parity between subgroups defined by ethnicity. Simultaneously, the model can be constrained to fulfill demographic parity between subgroups defined by gender. We term fairness in this situation \textit{independent group fairness}. On the other hand, one can consider all subgroups defined by intersections of sensitive attributes (e.g., ethnicity and gender), leading to \emph{intersectional group fairness}. 
A given algorithm can be \textit{independently group fair}, e.g., when considering race and gender in isolation, but not \textit{intersectionally group fair}, e.g., when considering intersections of racial and gender groups. For example, \citet{Buolamwini18}, showed how facial recognition software had a particularly poor performance for black women. 
This phenomenon, called \emph{fairness gerrymandering}, has been studied by~\citet{Kearns18}. Intersectional fairness is often considered ideal. However, it comes with major statistical and computational hurdles such as data scarcity at intersections of minority groups, and the potentially exponential number of subgroups. Indeed, current algorithms consist of either brute force enumeration or searching via a cost-sensitive classification problem, and intersectional groups are often empty with finite samples~\citep{Kearns18}. 
On the other hand, independent group fairness still provides a broad measure of fairness and is much easier to enforce.

We seek to { {\em design unifying statistically consistent strategies for group fairness and to clarify the relationship between the existing definitions.}}
Our main results and algorithms apply to arbitrary overlapping group definitions. 
Our contributions are summarized in the following.
\begin{itemize}
\item {\bf Probabiistic results}. We characterize the population optimal (also known as the Bayes-optimal) prediction procedure for multiclass classification, where all the metrics are general linear functions of the confusion matrix. We consider both overlapping (independent, gerrymandering) and non-overlapping (unrestricted, intersectional) group fairness.
\item {\bf Algorithms and statistical results.} Inspired by the population optimal, we propose simple plugin and weighted empirical risk minimization (ERM) approaches for algorithmically fair classification, and prove their consistency, i.e., the empirical estimator converges to the population optimal with sufficiently large samples. Our general approach recovers existing results for plugin and weighted ERM group-fair classifiers.
\item {\bf Comparisons.} We compare independent group fairness to the overlapping case. We show that
    intersectional fairness implies overlapping group fairness under weak conditions. However, the converse is not true, i.e., overlapping fairness may not imply intersectional fairness. This result formalizes existing observations on the dangers of gerrymandering.
\item {\bf Evaluation.}
Empirical results are provided to highlight our theoretical claims.
\end{itemize}
Taken together, our results unify and advance the state of the art with respect to the probabilistic, statistical, and algorithmic understanding of group-fair classification. The generality of our approach gives significant flexibility to the algorithm designer when constructing algorithmically-fair learners. 

\section{Problem Setup and Notation}
\label{section:framework}

Throughout the paper, we use uppercased bold letters to represent matrices, and lowercased bold letters to represent vectors. Let $e_i$ represent the $i$th standard basis whose $i$th dimension is 1 and 0 otherwise $e_i=(0,\cdots,1,\cdots,0)$. We denote $\vec 1$ as the all-ones vector with dimension inferred from context. Given two matrices $\bA,\bB$ of same dimension, $\ip{\bA,\bB} = \sum_{i,j} a_{ij}b_{ij}$ is the Frobenius inner product. For any quantity $q$, $\hat q$ denotes an empirical estimate. Due to limited space, proofs are presented in the appendix.

{\bf Group notation.}
We assume $M$ sensitive attributes, where each attribute is indicated by a group $\{\cA_m\}_{m \in [M]}$. For example, $\cA_1$ may correspond to race, $\cA_2$ may correspond to gender, and so on. Combined, the sensitive group indicator is represented by a $M$-dimensional vector $\ba \in \cA = \cA_1 \times \cA_2 \times \cdots \cA_M$. In other words, each instance is associated with $M$ subgroups simultaneously. 

{\bf Probabilistic notation.}
Consider the multiclass classification problem where $\cZ$ denotes the instance space and $\mathcal{Y} = \left[K\right]$ denotes the output space with $K$ classes. We assume the instances, outputs and groups are samples from a probability distribution $\mathbb{P}$ over the domain $\mathcal{Y}\times\mathcal{Z}\times\cA$. A dataset is given by $n$ samples $(y^{(i)}, z^{(i)}, a^{(i)}) \overset{\text{i.i.d}}{\sim} \mathbb{P}, i\in[n]$.
To simplify notation, let $\cX = \cZ \times \cA$, so $\bx = (\bz,\ba)$.
Define the set of randomized classifiers $\mathcal{H}_r=\{\bh:  \mathcal{X} \times \cA \to (\Delta^K) \}$, where $\Delta^q = \set{\mathbf{p}\in [0,1]^q: \sum_{i=1}^q p_i = 1}$ is the $q-1$ dimensional probability simplex. A classifier $\vect h$ is associated with the random variable $h\in [K]$ defined by $\P(h=k|\bx) = h_k(\bx)$. If $\vect h$ is deterministic, then we can write $\vect h(\bx) = e_{h(\bx)}$.

{\em Confusion matrices.}
For any multiclass classifier, let $\vect{\eta}(\bx) \in \Delta^{K}$ denote the class probabilities for any given instance $\bx$ and sensitive attribute $\ba$, whose $k$th element is the conditional probability of the output belonging to class $k$, i.e., $\eta_k(\bx) = \P(Y = k \mid X = \bx)$. The population confusion matrix is ${\bC}\in [0,1]^{K\times K}$, with elements defined for $k,\ell \in[K]$ as ${\bC}_{k,\ell} = \P(Y=k, h=\ell)$, or equivalently, 
\begin{align*}
{\bC}_{k,\ell} = \int_{\bx} \vect{\eta}_k(\bx)h_\ell(\bx)\,d\P(\bx). 
\label{eq:defC}
\end{align*}

{\em Group-specific confusion matrices.} Let $\cG$ represent a set of subsets of the instances, i.e., potentially overlapping partitions of the instances $\cX$. We leave $\cG$ as generic for now, and will specify cases specific to fairness in the following. Given any group $g \in \cG$, we can define the group-specific confusion matrix ${\bC^g}\in [0,1]^{K\times K}$, with elements defined for $k,\ell\in[K]$, where  
\begin{align*}
{\bC}^g_{k,\ell} = \int_{\bx} \vect{\eta}_k(\bx)h_{\ell}(\bx)\,d\P(\bx|\bx\in g).
\end{align*}
We will abbreviate the event $\{\bx \in g\}$ to simply $g$ when it is clear from context.
Let $\pi_{g} = \P(X\in g)$ be the probability of group $g$. It is clear that when the groups $\cG$ form a partition, i.e., $a \cap b = \emptyset \, \forall a, b \in \cG$ and $\bigcup_{g\in\cG} g = \cX$, the population confusion may be recovered by a weighted average of group confusions, $\bC = \sum_{g \in \cG} \pi_{g} \bC^{g}.$ 
Let $\omega_{k} = \P(Y=k) = \sum_{\ell} {\bC}_{k,\ell} $ be the probability of label $k$, and $\omega_{k}^g = \P(Y=k | X\in g) = \sum_{\ell} {\bC}^g_{k,\ell} $ be the probability of label $k$ given group $g$. 

{\bf The sample confusion matrix}
is defined as $\mat{\sConf}[\vect{h}] = \frac{1}{n} \sum_{i=1}^n \mat{\sConf}^{(i)}[\vect{h}]$, where
$\mat{\sConf}^{(i)}[\vect{h}] \in [0, 1 ]^{K\times K}$, and $\sConf_{k,\ell}^{(i)}[\vect{h}] = \indicator{y_i=k}h_\ell(\bx_i)$. Here, $\indicator{\cdot}$ is the indicator function, so $\sum_{k=1}^{K}\sum_{\ell =1}^{K}\sConf_{k,\ell}^{(i)}[\vect{h}]=1$.
{\em The empirical group-specific confusion matrices} $\wh{\mat C}^g$ are computed by conditioning on groups. In the empirical case, it is convenient to represent group memberships via indices alone, i.e., $\bx_i \in g$ as $i \in g$. 
We have $\mat{\sConf}^g[\vect{h}] = \frac{1}{|g|} \sum_{i\in g} \mat{\sConf}^{(i)}[\vect{h}]$.

{\bf Fairness constraints.}
Let $\gps$ represent the (potentially overlapping) set of groups across which we wish to enforce fairness. The following states our formal assumptions on $\gps$.
\begin{assumption}
$\gps$ is a function of the sensitive attributes $\cA$ only.
\label{ass:gps}
\end{assumption}
We will focus the discussion on common cases in the literature. These include non-overlapping (unrestricted, intersectional), and overlapping (independent, gerrymandering) group partitions.
\begin{itemize}
\item {\em Unrestricted case.} The simplest case is where the group is defined by a single sensitive attribute (when there are multiple sensitive attributes, all but one are ignored). These have been the primary settings addressed by past literature \citep{Hardt16, Narasimhan18, agarwal18}. Thus for some fixed $i \in [M]$, $g_{j} =\{(\bz, \ba)|a_i = j \}$, so $|\cG_\text{unrestricted}| = |A_i|$. In the special case of binary sensitive attributes, $|\cG_\text{unrestricted}| = 2$.
\item {\em Intersectional groups}. Here, the non-overlapping groups are associated with all possible combinations of sensitive features. Thus $g_\ba =\{(\bz, \ba')|\ba'=\ba\} \, \forall \ba \in \cA$ so $|\cG_\text{intersectional}| = \prod_{m \in M}|A_m|$. In the special case of binary sensitive attributes, $|\cG_\text{intersectional}| = 2^M$.
\item {\em Independent groups}. Here, the groups are overlapping, with a set of groups associated with each fairness attribute separately. It is convenient to denote the groups based on indices representing each attribute, and each potential setting. Thus $g_{i,j} =\{(\bz, \ba)|a_i = j \}$, so $|\cG_\text{independent}| = \sum_{m \in M}|A_m|$. In the special case of binary sensitive attributes, $|\cG_\text{independent}| = 2M$.
\item {\em Gerrymandering intersectional groups}. Here, group intersections are defined by any subset of the sensitive attributes, leading to overlapping subgroups. $\cG_\text{gerrymandering} = \{\{(\bz, \ba): \ba_I=\bs\}: I\subset[M],\, \bs \in \cA_I\}$  
where $\ba_I$ denotes $\ba$ restricted to the entries indexed by $I$. It is also the closure of $\cG_\text{independent}$ under intersection. As a result, $\cG_\text{intersectional} \subseteq \cG_\text{gerrymandering}$, and $\cG_\text{independent} \subseteq \cG_\text{gerrymandering}$. In the special case of binary sensitive attributes, $|\cG_\text{gerrymandering}| = 3^M$. 
\end{itemize}

{\bf Fairness metrics.}
We formulate group fairness by upper bounding a fairness violation function $\cV: \mathcal{H} \mapsto \mathbb{R}^J$ which can be represented as a linear function of the confusion matrices, i.e. $\cV(\vect{h})= \Phi(\mat C[{\vect h}], \confsh{\vect h})$ where $\forall j\in [J],\; \cV(\vect{h})_j =  \phi_j(\mat C[{\vect h}], \confsh{\vect h}) = \ip{\bU_j, \C} - \sum_{g\in\gps}   \ip{\bV_j^g, \C^g}$. This formulation is sufficiently flexible to include the fairness statistics we are aware of in common use as special cases. For example, demographic parity for binary classifiers \citep{Dwork2012} can be defined by fixing $\mat{C}_{0,0}^g+\mat{C}_{1,1}^g$ across groups. Equal opportunity \citep{Hardt2016} is recovered by fixing the group-specific true positives, using population specific weights, i.e.,
\begin{equation*}
\label{eq:dp}
\phi_\text{DP}^{\pm} = \pm (\mat{C}_{0,0}^g+\mat{C}_{1,1}^g
- \mat{C}_{0,0}+\mat{C}_{1,1})-\nu,
 \quad
\phi_\text{EO}^{\pm} =  \pm\left(\frac{1}{\omega_1^g}\mat{C}_{1,1}^g
- \frac{1}{\omega_1}\mat{C}_{1,1}\right)-\nu,
\end{equation*}
using both a positive and negative constraint to penalize both positive and negative deviations between the group and the population, and relaxation $\nu$. 

\paragraph{Performance metrics.}
We consider an error metric $\cE: \mathcal{H} \mapsto \mathbb{R}_+$ that is a linear function of the population confusion $\cE(\mathbf{h}) = \psi(\bC) = \langle \bD, \bC[{\vect h}] \rangle$.  This setting has been studied in binary classification~\citep{pmlr-v80-yan18b}, multiclass classification~\citep{narasimhan2015consistent},  multilabel classification~\citep{koyejo2015consistent}, and multioutput classification~\citep{wang2019consistent}. For instance, standard classification error corresponds to setting $\bD = 1- \bI$.
The goal is to learn the Bayes-optimal classifier with respect to the given metric, which, when it exists, is given by:
\begin{equation}
\vect{h}^* \in \argmin_{\vect{h}}\; \cE(\vect{h}) \;  \sut \; \cV(\vect{h}) \le \mathbf{0}.
\label{eq:Bayes}
\end{equation}

We denote the optimal error as $\cE^* = \cE(\vect{h}^*)$. We say a classifier $\vect{h}_N$ constructed using finite data of size $N$ is $\{\cE, \cV\}$-consistent if $ \cE(\vect{h}_n)\xrightarrow{ \mathbb{P} }\cE^*$ and $\cV(\vect h_n) \xrightarrow{\P} \mathbf{0}$, as $n \to \infty$. We also consider empirical versions of error $\hat\cE(\vect h) = \psi(\hmc[{\vect h}])$ and fairness violation $\wh\cV(\vect h) = \Phi(\hmc[{\vect h}\, \honfsh{\vect h})$.

\begin{table}[h]
\caption{Examples of multiclass performance metrics and fairness metrics studied in this manuscript.}
\label{table-metrics}
    \begin{center}
    \begin{tabular}{llll}
    \toprule
    {Metric} & $\psi(\mat{\conf})$ & {Fairness Metric} & $\phi(\mat{\conf}, \{ {\mat{\conf}}^g \}_g )$ \\\midrule
    Weighted Acc. & $\sum_{i=1}^{K}\sum_{j=1}^{K}b_{i,j}\conf_{i,j}$ &
    Demographic Parity & 
    $(\mat{C}_{0,0}^g+\mat{C}_{1,1}^g - \mat{C}_{0,0}+\mat{C}_{1,1})-\nu$
    \\[5pt]
    Ordinal Acc. & 
    $\sum_{i=1}^{K}\sum_{j=1}^{K}(1-\frac{1}{K-1} |i-j|)\conf_{i,j}$ 
    & 
    Equalized Opportunity &
    $\left(\frac{1}{\omega_1^g}\mat{C}_{1,1}^g - \frac{1}{\omega^g}\mat{C}_{1,1}\right)-\nu$
    \\[5pt]
    \bottomrule
    \end{tabular}
    \end{center}
\end{table}

\section{Bayes-Optimal Classifiers}
\label{sec:bayesopt}

In this section, we identify a parametric form for the Bayes-optimal group-fair classifier under standard assumptions. To begin, we introduce the following general assumption on the joint distribution.
\begin{assumption}[$\eta$-continuity]
\label{assumption1}
Assume $\mathbb{P}(\{ \vect{\eta}(\bx) = \mathbf{c} \}) = 0 \; \forall \mathbf{c} \in \Delta^K.$
Furthermore, let $Q= \vect{\eta}(\bx)$ be a random variable with density $p_{\eta}(Q)$, where $p_{\eta}(Q)$ is absolutely continuous with respect to the Lebesgue measure restricted to $\Delta^K$.
\label{ass:data}
\end{assumption}
This assumption imposes that the conditional probability as a random variable has a well-defined density. Analogous regularity assumptions are widely employed in literature on designing well-defined complex classification metrics and seem to be unavoidable (we refer interested reader to~\citet{pmlr-v80-yan18b,narasimhan2015consistent} for details).
Next, we define the general form of weighted multiclass classifiers, which are the Bayes-optimal classifiers for linear metrics.
\begin{definition}
\label{def:Min-Form}[\citet{narasimhan2015consistent}]
Given a loss matrix $\mat{W}\in \bR^{K\times K}$, a weighted classifier $\bh$ satisfies $h_i(\bx)>0$ only if $i \in \arg\min_{k\in[K]} \ip{\mat{W}_{k}, \mat{\eta(\bx)} }$.
\end{definition}

Next we present our first main result identifying the Bayes-optimal group-fair classifier.
\begin{theorem}
\label{thrm:Min-Form}
Under Assumption~\ref{ass:gps} and Assumption~\ref{ass:data}, if \eqref{eq:Bayes} is feasible (i.e., a solution exists), the Bayes-optimal classifier is given by $\mathbf{h}^*(\bx) = \mathbf{h}^*(\bz, \ba) = \beta_{\ba}\mathbf{h}_1(\bx) + (1-\beta_{\ba})\mathbf{h}_2(\bx),$ where $\beta_{\ba} \in (0,1), \forall \ba \in \cA$ and $\mathbf{h}_i(\bx)$ are weighted classifiers with weights $\{\{ \bW_{i, \ba} \}_{i \in \{1, 2\}}\}_{\ba \in \cA}$.
\label{thm:max_over_eta}
\end{theorem}
One key observation is that pointwise, the Bayes-optimal classifier can be decomposed based on intersectional groups $\cG_{\text{intersectional}} = \cA$, even when $\cG_\text{fair}$ is overlapping. This observation will prove useful for algorithms. 

\subsection{Intersectional group fairness implies overlapping group fairness}

Recent research~\cite{Kearns18} has shown how imposing overlapping group fairness using independent fairness restrictions can lead to violation of intersectional fairness, primarily via examples. This observation led to the term {\em fairness gerrymandering}. Here, we examine this claim more formally, showing that enforcing intersectional fairness controls overlapping fairness, although the converse is not always true, i.e., enforcing overlapping fairness does not imply intersectional fairness. We show this result for the general case of quasi-convex fairness measures, with linear fairness metrics recovered as a special case.
\begin{proposition}
For any $\gps$ that satisfies assumption~\ref{ass:gps}, suppose $\phi:[0,1]^{K\times K}\times[0,1]^{K\times K}\to \R_+$ is quasiconvex, 
$\phi(\C, \C^g) \leq 0 
 \, \forall g\in\cG_\text{intersectional}
\implies 
\phi(\C,\C^g)\leq 0 \,
\forall g\in\gps. $
The converse does not hold.
\label{section:comparison}
\end{proposition}
\begin{remark}
Note that the converse claim of Proposition~\ref{section:comparison}, does not apply to $\cG_\text{gerrymandering}$. Controlling the gerrymandering fairness violation implies control of the intersectional fairness violation, since $\cG_\text{intersectional} \subseteq \cG_\text{gerrymandering}$.
\end{remark}

\begin{algorithm}[t]
\caption{\name,
Group-fair classification with overlapping groups,
\label{alg:general}}
\KwIn{$\psi:[0,1]^{K\times K} \to[0,1],\, 
\Phi: [0,1]^{K\times K}\times([0,1]^{K\times K})^{\gps}\to[0,1]^J$}
\myinput{samples $\{(\bx_1,y_1),\ldots, (\bx_n, y_n)\}$.}
Initialize $\vec\blambda_1\in [0,B]^{J}$\;
\For{$t=1,\ldots, T$}{
$h^t \gets \mino_{h\in\cH}(\cL(h,\vec\blambda^t), z^n)$\;
$\vec\blambda^{t+1}\gets \update_t(\vec\blambda^t, \Phi(\hmc[h^t], \hconfsh{h^t})-\ve)$\;
}
$\bar{\vect{h}}^T \gets \frac{1}{T}\sum_{t=1}^T \vect h^t,\quad 
\bar{\vec\blambda}^T\gets \frac{1}{T}\sum_{t=1}^T\vec\blambda^t$\;
\vspace{0.03in}
\Return{$(\bar{\vect h}^T, \bar{\vec\blambda}^T)$}
\end{algorithm}

\section{Algorithms}
\label{section:algorithms}
Here we present \name, a general empirical procedure for solving \eqref{eq:Bayes}. 
The Lagrangian of the constrained optimization problem \eqref{eq:Bayes} is $\cL(\vect h, \blambda) = \cE(\vect h) + \blambda^\top\cV(\vect h)$ with empirical Lagrangian $\hat\cL(\vect h,\blambda) = \hat\cE(\vect h) + \blambda^\top(\cV(\vect h)-\ve)$, where $\ve$ is a buffer for generalization.
Our approach involves finding a saddle point of the Lagrangian. The returned classifiers will be probabilistic combinations of classifiers in $\cH$, i.e. the procedure returns a classifier in $\conv(\cH)$.
In the following, we first assume the dual parameter $\blambda$ is fixed, and describe the primal solution as a classification oracle. We consider both plugin and weighted ERM. In brief, the plugin estimator first proceeds assuming $\beeta(\bx)$ is known, then we {\em plugin} the empirical estimator $\hat\beeta(\bx)$ in its place. The plugin approach has the benefit of low computational complexity once fixed. On the other hand, the weighted ERM estimator requires the solution of a weighted classification problem in each round, but avoids the need for estimating $\hat\beeta(\bx)$.

\subsection{Weighted ERM Oracle}
\label{section:werm}

In the weighed ERM approach we parametrize $h:\cX \to[K]$ by a function class $\cF$ of functions $\bff:\cX\to\R^K$. The classification is the argmax of the predicted vector, $h(\bx) = \argmax_j(\bff(\bx)_j)$, so we denote the set of classifiers as $\cH^{werm} = \argmax\circ\cF$. 
The following special case of Definition 1 in \citep{ramaswamy2016convex} outlines the required conditions for weighted multiclass classification calibration. This is commonly referred to as cost-sensitive classification~\citep{agarwal18} when applied to binary classification.
\begin{definition}[$\bW$-calibration~\citep{ramaswamy2016convex}]
Let $\bW \in \bbR_+^{K\times K}$. A surrogate function $\bL: \bbR^K \to \bbR^K_+$ is said to be $\bW$-calibrated if
$$
\forall p \in \Delta^K: \inf_{\bu: \argmax(u) \notin \argmin_k (\bp^\top\bW)_k } \bp^\top\bL(\bu) > \inf_{\bu} \bp^\top\bL(\bu).
$$
\label{def:calibration}
\end{definition}
Note that the weights are sample (group) specific -- which, while uncommon, is not new, e.g., \citet{pires13}.
\begin{proposition}
The weighted ERM estimator for average fairness violation is given by:
$
h(\bx) = \argmax_j(\bff^*(\bx)_j), \;
\bff^* = \min_{\bff\in\cF} \hat L(f); \;$
where $\hat L(\bff) = \hat \bbE [\by^T \bL(\bff)]$ is a multiclass classification surrogate for the weighted multiclass error with group-dependent weights $\forall \ba \in \cA$
\begin{align}
    \bW(\bx) = 
    \left[\bD + \sum_{j=1}^J \blambda_j\bigg(\bU_j-\sum_{g\in\gps}\frac{\1_{\ba\in g}}{\hat\pi(g)}\bV_j^g\bigg)\right].
    \label{eq:weights}
\end{align}
\label{prop:wwerm}
\end{proposition}

\subsection{The Plugin Oracle}
\label{section:plugin}

The plugin hypothesis class are the weighted classifiers, identified by Theorem~\ref{thrm:Min-Form} as $\cH^{plg} = \{h(\bx)=\argmin_{j\in [K]}( \hat\beeta(\bx)^\top \bB(\bx))_j: \bB(\bx)\in\R^{K\times K}\}$. Here, we focus on the average violation case only. 
By simply-reordering terms, the population problem can be determined as follows.
\begin{proposition}
The plug-in estimator for average fairness violation is given by $\hat h(\bx) = \argmin_{k\in [K]} (\beeta(\bx)^\top \bW(\bx))_k$, where $\bW(\bx)$ is defined in \eqref{eq:weights}.
\label{prop:plugin}
\end{proposition}

\subsection{\name, a General Group-Fair Classification Algorithm}
We can now present \name, a general algorithm for group-fair classification with overlapping groups, as outlined in Algorithm~\ref{alg:general}. As outlined, our approach proceeds in rounds, updating the classifier oracle and the dual variable. Interleaved with the primal update is a dual update $\update_t(\blambda, \bv)$ via gradient descent on the dual variable. The resulting classifier is the average over the oracle classifiers.

{\bf Recovery of existing methods.}
When the groups are non-overlapping, \name\ with the Plugin oracle and projected gradient ascent update recovers FairCOCO ~\citep{Narasimhan18}. Similarly, when the groups are non-overlapping, and the labels are binary, \name\ with the weighted ERM oracle and exponentiated gradient update recovers FairReduction~\citep{agarwal18} (see also Table~\ref{tab:algs}). Importantly, \name\ enables a straightforward extension to overlapping groups.

\section{Consistency}
\label{section:consistency}
Here we discuss the consistency of the weighted ERM and the plugin approaches. 
For any class $\cH = \{h:\cX\to[K]\}$, denote $\cH_k = \{\1_{\{h(x)=k\}}:h\in \cH\}$. We assume WLOG that $\vc(\cH_1)=\ldots=\vc(\cH_K)$ and denote this quantity as $\vc(\cH)$. 
Next, we give a theorem relating the performance and satisfaction of constraints of an empirical saddle point to an optimal fair classifier. 

\begin{theorem}
\label{thm:saddle-main}
Suppose $\psi : [0,1]^{K\times K}\to [0, 1]$ and $\Phi:[0,1]^{K\times K}\times ([0,1]^{K\times K})^{\gps}\to [0,1]^J$ are $\rho$-Lipschitz w.r.t. $\|\cdot\|_{\infty}$.
Recall $\hat\cL(\vect h,\blambda)=\hat\cE(\vect h)+\blambda^\top(\hat\cV(\vect h)-\ve\vec 1)$.  Define $\gamma(n',\cH,\delta) = \sqrt{\frac{\VC(\cH)\log(n)+\log(1/\delta)}{n}}$. If $n_{\min} = \min_{g\in\gps} n_g,\, \ve = \Omega\left(\rho\gamma(n_{\min}, \cH, \delta)\right)$ then w.p. $1-\delta$:

If $(\bar{\vect h}, \bar{\vec\blambda})$ is a $\nu$-saddle point of $\max_{\blambda\in[0,B]^J}\min_{\vect h\in\conv\cH} \hat\cL(\vect h, \blambda)$, in the sense that $\max_{\blambda\in[0,B]^J} \hat\cL(\bar{\vect h}, \blambda)-\min_{\vect h\in\conv(\cH)}\hat\cL(\vect h, \bar{\vec\blambda})\leq\nu$, and $\vect h^*\in\conv(\cH)$ satisfies $\cV(\vect h^*)\leq 0$, then
\begin{equation*}
    \cE(\bar{\vect h})\leq \cE(\vect h^*)+\nu+\cO\left(\rho\gamma(n,\cH,\delta)\right)
    , \quad
    \|\cV(\bar{\vect h})\|_{\infty}\leq \frac{1+\nu}{B}+\cO\left(\rho\gamma(n_{\min}, \cH,\delta)\right)+\ve.
\end{equation*}
\end{theorem}
Thus, as long as we can find an arbitrarily good saddle point, which weighted ERM grants if $\cH^{werm}$ is expressive enough while having finite VC dimension, then we obtain consistency. A saddle point can be found by running a gradient ascent algorithm on $\blambda$ confined to $[0,B]^J$, which repeatedly computes $h^t = \argmin_{h\in\cH} \hat\cL(h, \blambda^t)$; the final $(\bar{\vect  h},\bar{\vec\blambda})$ are the averages of the primal and dual variables computed throughout the algorithm. 

Although Theorem~\ref{thm:saddle-main} captures the spirit of the argument for the plugin algorithm, it only applies naturally to the weighted ERM algorithm. This is because the plugin algorithm is solving a subtly different minimization problem: it returns $h^t$ as the \textit{population minimum}, \textit{if the estimated regression function $\hat\eta$ replaces the true regression function}. 
\begin{theorem}
\label{thm:plugcon}
With probability at least $1-\delta$, if projected gradient ascent is run as $\update_t(\vec\blambda, \bv) = \proj_{[0,B]^J}(\vec\blambda+\eta \bv)$ for $T$ iterations with step size $\eta= \frac{1}{B\sqrt{T}}$ and for $t=1,\ldots, T,\; h^t= \plugin(\hat\beeta, (\hat\pi_g)_{g\in\gps}, \psi, \Phi)$, letting $\rho = \max\{\|\psi\|_1, \|\phi_1\|_1,
\ldots, \|\phi_M\|_1\},\, \rho_g = \sum_{j=1}^J \|\bV^g_j\|_{\infty},\, \rho_{\all} = \|\bD\|_{\infty}+\sum_{j=1}^J \|\bU_j\|_{\infty},\,\Delta\beeta = \E\|\beeta(x)-\hat\beeta(x)\|_1, \check{n} = \min_{g\in\gps} n_g$, then
\begin{gather*}
\kappa := \cO\left(J\rho\sqrt{\frac{K^2\log(\check n) + \log(\frac{|\gps|K^2}{\delta})}{\check n}}\right) +\Delta\beeta \left(\rho_{\cX} + \sum_{g\in\gps}\frac{\rho_g}{\pi_g}\right) + \sqrt{\frac{\log(\frac{|\gps|}{\delta})}{n}}\sum_{g\in\gps}\frac{\rho_g}{\pi_g^2} \\
    \implies \cE_{\psi}(\bar{\vect h}^T) \leq \cE_{\psi}^*
    + \frac{JB}{\sqrt{T}}
    + \cO\left(BJ\kappa\right),
    \qquad \|\cV_{\phi}(\bar{\vect h}^T)\|_{\infty} \leq \frac{2J}{\sqrt{T}} + \cO\left(J\kappa\right). 
\end{gather*}
\end{theorem}

A key point in the presented analyses (for both procedures) is that the dominating statistical properties depend on the number of fairness groups. We note that $|\gps| \ll |\cG_\text{intersectional}| = |\cA|$ for the independent case, so this significantly improves results.  More broadly, we conjecture that the statistical bounds depend on $\min (|\gps|, |\cG_\text{intersectional}|)$, and leave the details to future work. We also note the statistical dependence on the size of the smallest group. This seems to be unavoidable, as we need an estimate of the group fairness violation in order to control it. To this end, group violations may be scaled by group size, which leads instead to a dependence on the VC dimension of $\gps$, improving statistical dependence with small groups at the cost of some fairness~\cite{Kearns18}. We expect that the bounds may be improved by a more refined analysis, or modified algorithms with stronger assumptions. We leave this detail to future work.

\begin{table}[t]
    \centering
    \begin{tabular}{ccc}
    \toprule
     & $\mino_{h\in\cH}(\cL(h,\blambda^t), z^n)$ & $\update_t(\blambda, \bv)$ \\\midrule
     FairReduction  &  $H \circ \argmin_{f\in\cF} \hat L(f)$ & $B\frac{\exp(\log\blambda-\eta_t \bv)}{B-\sum_{j=1}^M\lambda_j+\sum_{j=1}^M\exp(\log\lambda_i-\eta_t v_i)}$\\
     FairCOCO   & $\plugin(\hat\beeta, (\hat\pi_g)_{g\in\gps}, \psi, \Phi, \blambda^t)$ &
     $\proj_{[0,B]^M}(\blambda + \eta_t \bv)$
     \\\bottomrule
    \end{tabular}
    \vspace{0.03in}
    \caption{The oracles shown are $\plugin$ \eqref{eq:plugsol} and ERM on the reweighted $\hat L$  \eqref{eq:rwl}.  $H=[\argmax_{k\in[K]} (\cdot)_k]$ converts a function $\cX\to\R^K$ to a classifier. In FairCOCO, $\hat\beeta$ is estimated from samples $z^{1:n/2}=\{(x_1,y_1),\ldots, (x_{n/2}, y_{n/2})\}$ and all of the other probability estimates $(\hat\pi_g)_g$ and $\honfsh{h^t}$ are estimated from $z^{n/2:}=z^n\setminus z^{1:n/2}$.}
    \label{tab:algs}
\end{table}

\subsection{Additional Related Work}
\label{section:related}

Recent work by \citet{Foulds18, Kearns18} and \citet{hebert2018} were among the first to define and study intersectional fairness with respect to parity and calibration metrics respectively. \citet{Narasimhan18} provide a plugin algorithm for group fairness and generalization guarantees for the unrestricted case. 
\citep{Menon2018} considered Bayes optimality of fair binary classification where the sensitive attribute is unknown at test time, using an additional sensitive attribute regressor.  
\citet{2018cotterwoodwang} provide proxy-Lagrangian algorithm with generalization guarantees, assuming proxy constraint functions which are strongly convex, and argue that better generalization is achieved by reserving part of the dataset for training primal parameters and part of the dataset for training dual parameters. \citet{2018celis} provide an algorithm with generalization guarantees for independent group fairness based on solving a grid of interval constrained programs; their and \citet{Narasimhan18}'s work are most similar to ours.

\section{Experiments}
\label{section:experiments}
We consider demographic parity  as the fairness violation, i.e., $\phi_\text{DP}^{\pm} = \pm (\mat{C}_{0,0}^g+\mat{C}_{1,1}^g
- \mat{C}_{0,0}+\mat{C}_{1,1})-\nu,$ combined with 0-1 error $\psi(\mat C) = \mat C_{01} + \mat C_{10}$ as the error metric. All labels and protected attributes are binary or binarized. We use the following datasets (details in the appendix): (i) Communities and Crime, (ii) Adult census, (iii) German credit and (iv) Law school.

{\bf Evaluation Metric.}
We compute the "fairness frontier" of each method -- that is, we vary the constraint level $\nu$. 
We plot the fairness violation and the error rate on the train set and a test set. The fairness violation for demographic parity is defined by 
\begin{equation*}
    \text{fairviol}_\text{DP} = \max_{g\in\cG_\text{fair}} |\hmc^g_{0, 1}+\hmc^g_{1, 1} - \hmc_{0,1}-\hmc_{1,1}|
\end{equation*}
Observe that on the training set, it is always possible to achieve extreme points by ignoring either the classification error or the fairness violation.

{\bf Baseline: \texttt{Regularizer}} is a linear classifier implemented by using Adam to minimize logistic loss plus the following regularization function: 
\begin{align}
  \rho \sum_{j=1}^M\left(\frac{\sum_{i: (z_i)_j=1} \sigma(w^\top x_i)}{|\{i:(z_i)_j=1\}|}
    - \frac{\sum_{i=1}^n \sigma(w^\top x_i)}{n}\right)^2
    \label{eq:reg}
\end{align}
where $\sigma(r) = \frac{1}{1+e^{-r}}$ is the sigmoid function. This penalizes the squared differences between the average prediction probabilities for each group and the overall average prediction probability. Other existing methods we are aware of are either not applicable to overlapping groups, or are special cases of \name.

{\bf Experiment 1: Independent group fairness.}
We consider independent group fairness, defined by considering protected attributes separately.
Our results compare extensions of FairCOCO~\citep{Narasimhan18} and a  FairReduction~\citep{agarwal18}, existing special cases of \name\ using the plugin and weighted ERM oracles respectively. 
Results are shown in Figure~\ref{fig:plots}. We further present the differences in training time in \ref{tab:times}. On all datasets, the variants of \name\ are much more effective than a generic regularization approach.  
However, \texttt{Plugin} seems to violate fairness more often at test time -- perhaps this is due to the $\|\hat\eta-\eta\|_1$ term in the generalization bound in Theorem~\ref{thm:plugcon}. At the same time, \texttt{Plugin} is almost 2 orders of magnitude faster, since its $\mino$ essentially has a closed-form solution, while \texttt{Weighted-ERM} has to solve a new ERM problem in each iteration. 

\begin{figure}[t]
    \centering
    \includegraphics[width=\linewidth]{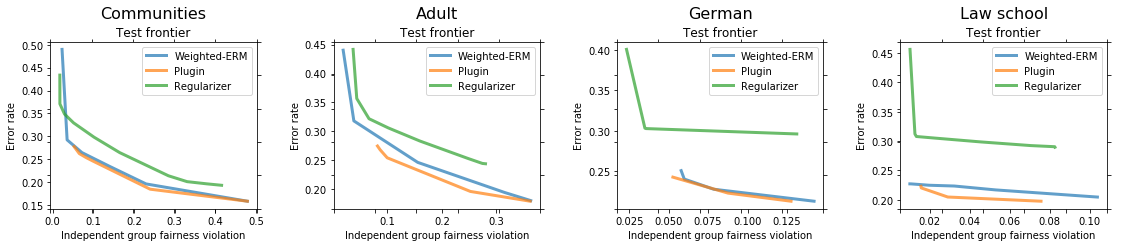}
    \caption{Experiments on independent group fairness, showing fairness frontier. The pareto frontier closest to the bottom left represent the best fairness/performance tradeoff.}
    \label{fig:plots}
\end{figure}
\begin{table}[t]
    \centering
    \caption{Average training times (averaged over the training sessions for each fairness parameter). The Plugin Oracle is significantly faster than other approaches.}
    \begin{tabular}{cccccccc}
    \toprule
    &\multicolumn{4}{c}{Independent}&
    \multicolumn{3}{c}{Gerrymandering}\\
    & C\& C & Adult & German & Law school & Adult & German & Law school\\\midrule
       \texttt{Weighted-ERM}  &  684.4 s & 424.0 s & 187.0 s & 68.6 s & 817.0 s&  40.4 s & 49.4 s\\
    \texttt{Plugin} & 11.5 s & 8.5 s &4.4 s &3.8 s & 699.8 s& 13.0 s & 17.7 s\\
    \texttt{Regularizer} & 75.4 s & 87.4 s & 35.2&  68.0 s & N/A & N/A & N/A\\
    \texttt{Kearns et al.}  &  N/A & N/A & N/A & N/A & 2213.7 & 821.5 s & 1674.4 s\\\bottomrule
    \end{tabular}
    \label{tab:times}
\end{table}
\begin{figure}[th!]
\centering
    \includegraphics[width=0.9\linewidth]{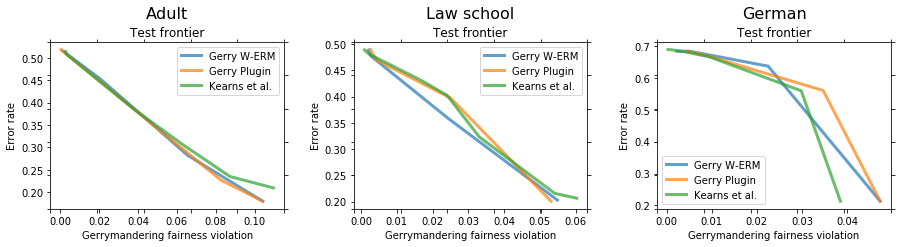}
    \caption{Experiments on gerrymandering group fairness. The pareto frontier closest to the bottom left represent the best fairness/performance tradeoff.}
    \label{fig:gerry}
\end{figure}

{\bf Experiment 2: Gerrymandering group fairness.}
Unfortunately, intersectional fairness is not statistically estimable in most cases as most intersections are empty. As a remedy, \citep{Kearns18} propose max-violation fairness constraints over $\cG_\text{gerrymandering}$, where each group is weighed by group size, i.e., $
\max_{g\in\cG_\text{gerrymandering}} \frac{|g|}{n}|\hmc^g_{0, 1}+\hmc^g_{1, 1} - \hmc_{0,1}-\hmc_{1,1}|
$, so empty groups are removed, and small groups have relatively low influence unless there is a very large fairness violation. We denote the approach of \citet{Kearns18} as \texttt{Kearns et al.} This approach is closely related to \texttt{Weighted-ERM} but searches for the maximally violated group by solving a cost-sensitive classification problem and uses fictitious play between $\lambda$ and $\vect h$. For the \texttt{Plugin} and \texttt{Weighted-ERM} approaches, we optimize the cost function directly using gradient ascent, precomputing the gerrymandering groups present in the data. 
Results are shown in Figure~\ref{fig:gerry}. We further present the differences in training time in Table~\ref{tab:times}. 
The results are roughly equivalent in terms of performance, however, both the \texttt{Weighted-ERM} and \texttt{Plugin} approach are 1-2 orders of magnitude faster than \texttt{Kearns et al.}

\section{Conclusion}
This manuscript considered algorithmic fairness across multiple overlapping groups simultaneously. Using a probabilistic population analysis, we present the Bayes-optimal classifier, which motivates a general-purpose algorithm, \name. Our approach unifies a variety of existing group-fair classification methods and enables extensions to a wide range of non-decomposable multiclass performance metrics and fairness measures. 
Future work will include extensions beyond linear metrics, to consider more general fractional and convex metrics. We also wish to explore more complex prediction settings beyond classification.

\bibliographystyle{plainnat}

\newpage
\appendix

\section*{Appendix}

\section{Bayes optimal}
\begin{customthm}{\ref{thrm:Min-Form}}
Under Assumption~\ref{ass:gps} and Assumption~\ref{ass:data}, if \eqref{eq:Bayes}, i.e., 
\begin{equation*}
\vect{h}^* \in \argmin_{\vect{h}}\; \cE(\vect{h}) \;  \sut \; \cV(\vect{h}) \le \mathbf{0},
\end{equation*}
is feasible (i.e., a solution exists), the Bayes-optimal classifier is given by $\mathbf{h}^*(\bx) = \mathbf{h}^*(\bz, \ba) = \beta_{\ba}\mathbf{h}_1(\bx) + (1-\beta_{\ba})\mathbf{h}_2(\bx),$ where $\beta_{\ba} \in (0,1), \forall \ba \in \cA$ and $\mathbf{h}_i(\bx)$ are weighted classifiers with weights $\{\{ \bW_{i, \ba} \}_{i \in \{1, 2\}}\}_{\ba \in \cA}$.
\end{customthm}
\begin{proof}
The key idea of the proof is to exploit the problem representation in terms of confusion matrices. The proof has two main steps (i) population analysis for feasible confusion matrices, and (ii) plug-in of the classifiers that achieve the Bayes optimal confusion.

{\bf Confusion space.}
As the first step, let $\cC^g = \{\bC^g(\bh) \, | \, \bh \in \cH\}$ be all group $g$ specific confusion matrices, and let $\cC_{\cG_\text{fair}} = \prod_{g \in \cG_\text{fair}} \cC^g$ be the product space of all confusion matrices corresponding to fair groups associated with a given instance of the problem. Similarly, let $\cC_{\cA} = \prod_{g \in \cG_\text{intersectional}} \cC^g$ be the product space of all confusion matrices corresponding to intersectional groups. A standard property of confusion matrices is that each $\cC^g$ is a convex set~\cite{narasimhan2015consistent, Narasimhan18, wang2019consistent}. Thus, each $\bC \in \cC^g$ can be described as a mixture of two boundary points, i.e., 
\begin{equation*}
\forall \, \bC \in \cC^g \, \exists \bC^1, \bC^2 \in \partial\cC^g, \, \beta \in [0, 1],\, 
\sut \, \bC = \beta \bC^1 + (1-\beta) \bC^2
\end{equation*}
Another useful fact is that all confusion matrices on the boundary can be achieved by a weighted classifier~\cite{narasimhan2015consistent, Narasimhan18, wang2019consistent}. This fact follows from the convexity of the set $C^g$, and is simply a dual representation -- via support functions, i.e., 
\begin{equation*}
\forall \, \bC \in \partial \cC^g, \, \exists \bW \, 
\sut \, \bC = \text{Conf}^g(\bh^*), \text{ where } \bh^* \in \underset{\bh \in \cH}{\argmax} \ip{\bW, \text{Conf}^g(\bh)},
\end{equation*}
and where, for notation clarity, we have  $\text{Conf}(\bh)$ as the confusion matrix of classifier $\bh$, and $\text{Conf}^g(\bh)$ as the group-restricted confusion matrix. Further, the solution $\bh^*$ can be represented as a weighted classifier (Definition~\ref{def:Min-Form}) ~\cite{Narasimhan18, wang2019consistent}. 

{\bf Population confusion problem.} Recall that the population confusion can be decoposed into their intersectional counterparts $\bC = \sum_{a \in \cG_\text{intersectional}} \bbP(a) \bC^a$. Similarly, each overlapping group confusion can be decomposed using the intersection confusions as $\bC^g \in \cC_{\cG_\text{fair}}$, $\bC^g = \sum_{a \in \cG_\text{intersectional}} \bbP(a|g) \bC^a$.

As the overall metric is a function of confusion matrices only, we can re-state \eqref{eq:Bayes} as the equivalent confusion problem (with slight abuse of notation) for any $\cG_\text{fair}$ as: 
\begin{gather*}
\bC^*, \{\bC^{g, *}\} = \argmin\; \psi(\bC) \;  \sut \; \Phi(\bC, \{\bC^{g}\}) \le \mathbf{0},\\
\bC = \sum_{a \in \cG_\text{intersectional}} \bbP(a) \bC^{a}\\
\bC^{g} = \sum_{a \in \cG_\text{intersectional}} \bbP(a|g) \bC^{a}\\
\bC^{a} = \text{Conf}^a(\bh).
\end{gather*}

After substituting the population $\bC$ and the group confusions $\bC^{g}$ with the presented linear functions of $\bC^{a}$, this is equivalent to the problem
\begin{equation*}
\{\bC^{a, *}\} = \argmin\; \psi( \{\bC^{a}\} ) \;  \sut \; \Phi(\{\bC^{a}\}) \le \mathbf{0},\quad
\bC^{a} = \text{Conf}^a(\bh).
\end{equation*}
Here, we have used the linearity of the cost functions $\psi$ and $\Phi$, and the linearity of the confusion matrix decompositions into intersectional confusion matrices. 

{\bf Putting it together.}
The final step is noting that a solution, if it exists, can be represented by feasible intersectional confusion matrices $\{\bC^{a, *}\}$, and in turn, each intersectional confusion matrix can be recovered as a weighted average of two intersectional boundary confusion matrices. Thus the corresponding classifiers can be recovered by a mixture of two weighted classifiers.
\end{proof}

\section{Independent vs. intersectional group fairness}

\begin{customprop}{\ref{section:comparison}}
For any $\gps$ that satisfies assumption~\ref{ass:gps}, suppose $\phi:[0,1]^{K\times K}\times[0,1]^{K\times K}\to \R_+$ is quasiconcave in its second argument, 
$\phi(\C, \C^g) \leq 0 
 \, \forall g\in\cG_\text{intersectional}
\implies 
\phi(\C,\C^g)\leq 0 \,
\forall g\in\gps. $
The converse does not hold.
\end{customprop}
\begin{proof}
(For the forward direction)

Recall that $f$ is {\em quasiconcave} if $f(\sum_i \lambda_i z_i) \le  \max_i \{f(z_i)\}$. 
When $\phi$ is quasiconvex, for any $\cG_\text{fair}$, we can compute $\phi(\bC, \bC^g) = \phi(\bC, \sum_{a \in \cG_\text{intersectional}} \lambda_a \bC^a) \le \max_{a \in \cG_\text{intersectional}} \phi(\bC, \bC^a)$, where $\lambda_a$ are linear weights (corresponding to inclusion probabilities).

Since $\phi(\bC, \bC^a) \le 0$ by the claim, it follows that $\phi(\C, \C^a) \leq 0 
 \, \forall a\in\cG_\text{intersectional}
\implies 
\phi(\C,\C^g)\leq 0 \,
\forall g\in\gps. $
\end{proof}

{\bf Converse.}
Though the above applies to any quasiconcave metric, in this manuscript we mainly consider linear metrics. As a corollary, intersectional group fairness with respect to common fairness metrics such as demographic parity or equal opportunity implies independent group fairness. A simple xor-like example from \citep{Kearns18} shows that the converse is not true. 

We provide another counterexample to the converse, showing a gap between independent and intersectional demographic parity (DP) group fairness, on an example with more realistic structure.
\begin{example}
Let $A_1, A_2, A_3$ be binary attributes and $\{A_m\}$ denote the event $\{A_m=1\}$. If  $\P(Y)=\P(A_1)=\P(A_2)=\P(A_3)=0.5$, $A_1,A_2,A_3$ are both independent and conditionally independent given $Y$, and $\P(A_m\mid Y) = 0.6$, then for every $P,N\subset\{1,2,3\}$ with $P\cap N=\emptyset$
\begin{equation*}
\P(Y\mid \cap_{i\in P} A_i, \cap_{j\in N} \bar{A}_j)
 = 0.5(1.2)^{|P|}(0.8)^{|N|}.
\end{equation*}
\end{example}
\begin{proposition}
An optimal (DP) intersectionally fair $\hat Y$ has, over every possible subgroup $G=\cap_{i\in P} A_i \cap_{j\in N} A_j,\; 
\P(\hat Y\mid G) = 0.384 = 0.5(1.2)^2(0.8)$ and has an error of $0.148$. 

On the other hand, an optimal (DP) independently fair classifier has  $\P(\hat Y\mid A_1,A_2,A_3) = 0.464,\,\P(\hat Y\mid A_i, A_j, \bar{A}_k) = 0.576,\,\P(\hat Y\mid A_i, \bar{A}_j, \bar{A}_k) = 0.384,\,\P(\hat Y\mid \bar{A}_i, \bar{A}_j, \bar{A}_k) = 0.656$ and has an error of $0.1$. 
\end{proposition}
Interestingly, even though $\P(Y\mid A_1, A_2, A_3) = 0.864$ and $\P(Y\mid \bar{A}_1,\bar{A}_2,\bar{A}_3) = 0.256$ have the highest and lowest probabilities, the reverse is true of the predictor $\hat Y$ -- it sacrifices accuracy on these groups to obtain higher accuracy on mixed positive/complement intersections.

Here we set up and discuss the example in \ref{section:comparison} in more detail. First we begin with a rigorous and more general description of the structure of the example -- here, one can think of a binary attribute as being synonymous with a partition with two sections. The first section corresponds to individuals with a value of 1 for that attribute and the other section to those with a value of 0.
\begin{assumption}[Independence]
\label{assump:ind}
Assume that the binary attributes $A_1, A_2,\ldots, A_M$ and label $Y$
satisfy:
\begin{enumerate}
\item $A_1,\ldots, A_M$ are independent.
\item $A_1,\ldots, A_M$ are independent conditioned on $Y$.
\end{enumerate}
\end{assumption}
In the following, when $A_j$ is used to denote an event inside a 
probability, it refers to the event $\{A_j=1\}$. $\bar{A_j}$ refers to 
the event $\{A_j=0\}$. We also use the notation $A_j = A_j^1$ and 
$\bar{A_j} = A_j^0$.
\begin{proposition}
\label{prop:basic}
For every $j=1,\ldots, M,\;$ define $q_j = P(A_j\mid Y)$ and $a_j = 
P(A_j)$. Then, under Assumption~\ref{assump:ind}, for any index set
$J=\{j_1,j_2,\ldots, j_{J}\}$ and $(b_j)_{j\in J}\in\{0,1\}^{J}$,
\begin{equation*}
P(Y\mid A_j^{b_j},\, j\in J) =
\prod_{k=1}^J \left(\frac{q_{j_k}}{a_{j_k}}\right)^{b_k}
\left(\frac{1-q_{j_k}}{1-a_{j_k}}\right)^{1-b_k}
\end{equation*}
\end{proposition}
\begin{proof}
\begin{align*}
P(Y\mid A_{j_1}^{b_1},\ldots, A_{j_J}^{b_J}) &= \frac{P(Y, A_{j_1}^{b_{1}},
\ldots,A_{j_J}^{b_{J}})}{P(A_{j_1}^{b_{1}},\ldots, A_{j_J}^{b_{J}})}
 \\
&= P(Y)\prod_{k=1}^J \frac{P(A_{j_k}^{b_k}\mid Y, A_{j_1}^{b_{1}},\ldots,
A_{j_{k-1}}^{b_{k-1}})}{P(A_{j_k}^{b_k} \mid A_{j_1}^{b_1},\ldots,
A_{j_{k-1}}^{b_{k-1}})} \\
&= P(Y)\prod_{k=1}^J \frac{P(A_{j_k}^{b_k}\mid Y)}{P(A_{j_k}^{b_k})} \\
&= P(Y)\prod_{k=1}^J \left(\frac{q_{j_k}}{a_{j_k}}\right)^{b_k}
\left(\frac{1-q_{j_k}}{1-a_{j_k}}\right)^{1-b_k}.
\end{align*}
The third line follows by independence, Assumption~\ref{assump:ind}.
\end{proof}
The idea behind the above proposition is that with the independence 
assumption~\ref{assump:ind}, the structure of $P(Y\mid A_1^{b_1},\ldots,
A_M^{b_M})$ is such that we have $P(Y)$ scaled either by $q_j/a_j$ or 
$(1-q_j)/(1-a_j)$ depending on whether we are in $A_j$ or $\bar{A_j}$.
This in a sense makes the effects of protected attributes ``pile on.''
If we assume WLOG that $q_j/a_j\geq 1$, then $(1-q_j)/(1-a_j)\leq 1$.
\begin{example}
\label{ex:3}
Suppose that $M=3,\; P(Y)=0.5$,
 and for every $j=1,2,3,\; a_j = P(A_j) = 0.5$ and 
$q_j = P(A_j\mid Y) = 0.6$. (This is possible because for every $J,\, 
0\leq P(Y\mid A_j,\, j\in J)\leq 1$, aka is a well defined probability.)
Applying Proposition~\ref{prop:basic} noting $\frac{q_j}{a_j} = 1.2,\,
\frac{1-q_j}{1-a_j}=0.8$, 
\begin{align*}
& P(Y\mid A_1)=P(Y\mid A_2) = P(Y\mid A_3) = 0.5\cdot 1.2 = 0.6, \\
& P(Y\mid \bar{A_1})=P(Y\mid \bar{A_2})=P(Y\mid \bar{A_3}) = 0.5\cdot 0.8
= 0.4, \\
& P(Y\mid A_{1}, A_2)=P(Y\mid A_1, A_3) = P(Y\mid A_2, A_3) = 0.5\cdot
(1.2)^2 = 0.72 \\
&\forall 1\leq j,k\leq 3,
 \quad P(Y\mid A_j, \bar{A_k}) = 0.5\cdot1.2\cdot0.8= 0.48 \\
&\forall 1\leq j,k\leq 3,
 \quad P(Y\mid \bar{A_j}, \bar{A_k}) = 0.5\cdot0.8\cdot0.8= 0.32 \\
& P(Y\mid A_1, A_2, A_3) = 0.5\cdot(1.2)^3= 0.864 \\
&\forall 1\leq i,j,k\leq 3,
\quad P(Y\mid A_i, A_j, \bar{A_k}) = 0.5\cdot(1.2)^2\cdot0.8= 0.576 \\
&\forall 1\leq i,j,k\leq 3,
 \quad P(Y\mid A_i, \bar{A_j}, \bar{A_k}) = 0.5\cdot1.2\cdot(0.8)^2= 
0.384 \\
& P(Y\mid \bar{A_1}, \bar{A_2}, \bar{A_3}) = 0.5\cdot(0.8)^3= 0.256 \\
\end{align*}.
\end{example}
\begin{fact}
\label{fact:intopt}
Assuming Assumption~\ref{assump:ind} and the accuracy metric,
 the optimal intersectionally
fair predictor $\hat Y$ assigns the probabilities
\begin{equation*}
\forall b\in\{0,1\}^M,\; P(\hat Y\mid A_1^{b_1},\ldots, A_M^{b_M})
= \wmed_{A}
\left\{P(Y)\prod_{j=1}^M\left(\frac{q_j}{a_j}\right)^{b_j}
\left(\frac{1-q_j}{1-a_j}\right)^{1-b_j}\right\}
\end{equation*}
where the weighted median $\wmed_A$ of a set of $2^M$ numbers 
$\{r_{b^1}\leq \ldots \leq r_{b^{2^M}} : b^i\in\{0,1\}^M\}$ is
\begin{equation*}
r_{b^{i^*}},\; i^* = \min\{i\in\N:\sum_{k\geq i} P(A_1^{b^k_1},\ldots, 
A_M^{b^k_M}) \geq 0.5\}.
\end{equation*}
\end{fact}
\begin{proof}[(Proof sketch)]
By thinking about it (or taking subgradient of $\E|Y-\hat Y|$), since 
we have the freedom to pick any constant to be the one to 
assign to every $P(\hat Y\mid A_1^{b_1},\ldots, 
A_M^{b_M})$, we get the weighted median formula.
\end{proof}
\begin{fact}
\label{fact:gap}
In example~\ref{ex:3}, using Fact~\ref{fact:intopt} (an)
optimal intersectionally fair predictor 
assigns $P(\hat Y\mid A_1^{b_1}, A_2^{b_2}, A_3^{b_3}) = 0.384$ and 
has an error of 
\begin{equation*}
\frac{1}{8}\left(|0.864-0.384|+3\cdot|0.576-0.384| + 
|0.256-0.384|\right) = 0.148.\end{equation*}
On the other hand, an optimal independently group fair predictor
assigns
\begin{align*}
& P(Y\mid A_1, A_2, A_3) = 0.5\cdot(1.2)^3= 0.464 \\
&\forall 1\leq i,j,k\leq 3,
\quad P(Y\mid A_i, A_j, \bar{A_k}) = 0.5\cdot(1.2)^2\cdot0.8= 0.576 \\
&\forall 1\leq i,j,k\leq 3,
 \quad P(Y\mid A_i, \bar{A_j}, \bar{A_k}) = 0.5\cdot1.2\cdot(0.8)^2= 
0.384 \\
& P(Y\mid \bar{A_1}, \bar{A_2}, \bar{A_3}) = 0.5\cdot(0.8)^3= 0.656.
\end{align*}
This predictor has an error of $\frac{1}{8}\left(|0.864-0.464|+|0.256-
0.656|\right)=0.1$. This is strictly less than the optimal intersectional
error $0.148$, i.e. there is a gap.
\end{fact}
\begin{proof}
By basically the same argument as for the intersectional case, it is optimal
to have $P(\hat Y\mid A_1)=P(\hat Y\mid \bar{A_1})$ be the median of 
$P(Y\mid A_1), P(Y\mid \bar{A_1})$. Now we just need to verify that 
$\hat Y$ as defined above is independently group fair. \\
\begin{align*}
P(Y\mid A_i) &= \frac{1}{4}\left(P(Y\mid A_i, A_j, A_k) + P(Y\mid A_i,
\bar{A_j}, A_k) + P(Y\mid A_i, A_j, \bar{A_k}) + P(Y\mid A_i, \bar{A_j},
\bar{A_k})\right)\\
& = \frac{1}{4}(0.464+2(0.576)+0.384) = 0.5\\
P(Y\mid \bar{A_i}) &= \frac{1}{4}\left(P(Y\mid \bar{A_i}, A_j, A_k)
 + P(Y\mid \bar{A_i}, \bar{A_j}, A_k) + P(Y\mid \bar{A_i}, 
A_j, \bar{A_k}) + P(Y\mid A_i, \bar{A_j}, \bar{A_k})\right)\\
& = \frac{1}{4}(0.576+2(0.384)+0.656) = 0.5.
\end{align*}
Since $i\in\{1,2,3\}$ is arbitrary independent group fairness is 
satisfied.
\end{proof}

\section{Consistency and Generalization}
\newcommand{\cir}{\tikz[baseline={(0,-0.07)}]\draw[black,fill=black] (0,0) circle (.3ex);}
\newcommand{\ohm}[1]{\overline{\hat{\mat{#1}}}}
\begin{customthm}{\ref{thm:plugcon}}
With probability at least $1-\delta$, if projected gradient ascent is run ($\update_t(\blambda, v) = \proj_{[0,B]^J}(\blambda+\eta v)$) for $T$ iterations with step size $\eta= \frac{1}{B\sqrt{T}}$ and for $t=1,\ldots, T,\; h^t= \plugin(\hat\beeta, (\hat\pi_g)_{g\in\gps}, \psi, \Phi)$, letting $\rho = \max\{\|\psi\|_1, \|\phi_1\|_1,
\ldots, \|\phi_J\|_1\}$, then
\begin{align*}
    &\cU_{\psi}(\bar{\vect h}^T) \leq \cU_{\psi}^*
    + \frac{JB}{\sqrt{T}}
    + ((1+J)B+1)\rho\left(4\sqrt{\frac{K^2\log(2n_{\min})}{n_{\min}}} + \sqrt{\frac{\log(2(1+|\gps|)K^2/\delta)}{n_{\min}}}\right) \\
    &\qquad\qquad\qquad\qquad + \E \|\beeta(x)-\hat\beeta(x)\|_1B\left(\rho_{\cX }+\sum_{g\in\gps}+ \frac{\rho_g }{\pi_g}\right) + 2\sqrt{\frac{\log(|\gps|/\delta)}{n}}\sum_{g\in\gps}\frac{\rho_g B}{\pi_g^2}\\
    &\|\cV_{\Phi}(\bar{\vect h}^T)\|_{\infty} \leq \frac{2J}{\sqrt{T}}
    +4(4(1+J)+1)\rho\left(\sqrt{\frac{K^2\log(2n_{\min})}{n_{\min}}} + \sqrt{\frac{\log(2(|1+|\gps|)K^2/\delta)}{n_g}}\right) \\
    &\qquad\qquad\qquad\qquad + 4\E \|\beeta(x)-\hat\beeta(x)\|_1\left(\rho_{\cX}+\sum_{g\in\gps}\frac{\rho_g}{\pi_g}\right) + 8\sqrt{\frac{\log(|\gps|/\delta)}{n}}\sum_{g\in\gps}\frac{\rho_g}{\pi_g^2}.
\end{align*}
\end{customthm}
\begin{proof}
First step is to extract the error incurred by plugging in $\hat\beeta$ rather than $\beeta$. Denoting $\hat h = \plugin(\hat\beeta, (\hat\pi_g)_g, \psi, \Phi, \blambda)$ and $n_g=|\{i:x_i\in g\}|$ so that $\hat\pi_g = \frac{n_g}{n}$,
\begin{equation*}
    \hat h(x) = 
    \argmin_{k\in\{1,\ldots, K\}} \bigg\{\hat\beeta(x)^\top\bigg[\bD + \sum_{l=1}^J \lambda_l\big(\bU_l-\sum_{g\in\gps}\frac{\1_{x\in g}}{\hat\pi_g}\bV_l^g\big)\bigg]\bigg\}_k.
\end{equation*}
Denote $h = \plugin(\beeta, (\pi_g)_g, \psi, \Phi, \blambda)$. We quantify the discrepancy. Define $\hat k = \hat h(x)$ and $k^* = h(x)$. Also, define 
\begin{equation*}
  \vec M =  \bD + \sum_{l=1}^J \lambda_l\bigg(\bU_l-\sum_{g\in\gps}\frac{\1_{x\in g}}{\hat\pi_g}\bV_l^g\bigg).
\end{equation*}
\begin{align*}
    (\beeta(x)^\top \vec M)_{\hat k}
    - (\beeta(x)^\top \vec M)_{k^*} &=
(\hat \beeta(x)^\top \vec M)_{\hat k}
+[(\beeta(x)-\hat\beeta(x))^\top \vec M]_{\hat k}
    - (\beeta(x)^\top \vec M)_{k^*} \\
&\leq (\hat \beeta(x)^\top \vec M)_{k^*}
+[(\beeta(x)-\hat\beeta(x))^\top \vec M]_{\hat k}
    - (\beeta(x)^\top \vec M)_{k^*} + \xi\\
&=(\beeta-\hat\beeta)^\top\vec M(e_{\hat k}-e_{k^*}) +\xi \leq \|\beeta-\hat\beeta\|_1\left(\sum_{g\in\gps}\frac{\rho_g }{\pi_g}+\rho_{\cX}\right)B + \xi
\end{align*}
where $\rho_g = \sum_{l=1}^J\| \bV^g_l\|_{\infty},\, \rho_{\cX} = \|\bD\|_{\infty} + \sum_{l=1}^J \|\bV_l\|_{\infty}$ and $\xi = 2\sqrt{\frac{\log(|\gps|/\delta)}{n}}\sum_{g\in\gps}\frac{\rho_g B}{\pi_g^2}$ -- we are considering the fact that $|\pi_g-\hat\pi_g|\leq \sqrt{\frac{\log(2|\gps|/n)}{n}}$ for every $g\in\gps$ with probability $1-\delta/2$. 
Taking expectation, we arrive at 
\begin{equation}
\label{eq:gen1}
    \cL(\mat C(\hat h), \blambda)-\cL(\mat C(h), \blambda) \leq \E \|\beeta(x)-\hat\beeta(x)\|_1\left(\sum_{g\in\gps}\frac{\rho_g}{\pi_g}+\rho_{\cX}\right)B + 2\sqrt{\frac{\log(|\gps|/\delta)}{n}}\sum_{g\in\gps}\frac{\rho_g B}{\pi_g^2}.
\end{equation}
By standard subgradient descent/online learning analysis, if the stepsize $\eta=  1/(B\sqrt{T})$ is used, 
\begin{equation*}
    \frac{1}{T}\max_{\blambda\in [0,B]^{2M}}
    \sum_{t=1}^T \hat\cL(h^t, \blambda)
    - \frac{1}{T}\sum_{t=1}^T \hat\cL(h^t, \blambda^t) \leq \frac{JB}{\sqrt{T}}
\end{equation*}
because $\cL(h, \cdot)$ is concave and $\sqrt{J}$-Lipschitz (all fairness violations assumed to be in $[0,1]$) and the $\ell_2$ radius of $[0,B]^{J}$ is $\sqrt{J}B$. 

Now we show how good of a saddle point $\left(\frac{1}{T}\sum_{t=1}^T \vect h^t, \frac{1}{T}\sum_{t=1}^T \blambda^t\right)=: (\bar{\vect h}^T, \bar{\blambda}^T)$ for the population problem. By convexity of $\cL$ in the first argument, 
\begin{equation*}
 \frac{1}{T}\max_{\blambda\in [0,B]^{M}}
    \sum_{t=1}^T \hat\cL(h^t, \blambda) 
\geq \max_{\blambda\in[0,B]^{M}} \hat\cL(\bar{\vect h}^T, \blambda).
\end{equation*}
Using equation~\ref{eq:gen1} and the fact that $h^t$ is the minimizer of $\cL(\mat C[h], \blambda^t)$, but using $\hat\beeta$ instead of $\beeta$, 
\begin{align*}
    \frac{1}{T}\sum_{t=1}^T \hat\cL(h^t, \blambda^t)& \leq 
     \frac{1}{T}\sum_{t=1}^T \cL(h^t, \blambda^t) + \hat\cL(h^t, \blambda^t) - \cL(h^t, \blambda^t)\\
     &\leq \frac{1}{T}\sum_{t=1}^T \min_{h:\cX\to[0,1]} \cL(h, \blambda^t) +
     \hat\cL(h^t, \blambda^t) - \cL(h^t, \blambda^t)\\
     &\qquad\qquad\qquad
     +B(\rho_{\cX}+ \sum_{g\in\gps}\frac{\rho_g }{\pi_g})\E\|\beeta(x)-\hat\beeta(x)\|_1 + \xi\\
     &\leq \min_{ h:\cX\to[0,1]}\cL( h, \overline{\blambda}^T) + 4(1+J)B\rho\left(\sqrt{\frac{K^2\log(K)\log(2n_{\min})}{n_{\min}}} + \sqrt{\frac{\log(2(|1+|\gps|)K^2/\delta)}{n_{\min}}}\right)\\
     &\qquad\qquad\qquad
     +B(\rho_{\cX} \sum_{g\in\gps}\frac{\rho_g }{\pi_g})\E\|\beeta(x)-\hat\beeta(x)\|_1 + \xi
\end{align*}
where the middle term is from Lemma~\ref{lem:gen}. Let us absorb the error terms into $\gamma$. Now we can write:
\begin{equation*}
    \max_{\blambda\in[0,B]^{J}} \hat\cL(\bar{\vect h}^T, \blambda) - 
    \min_{h:\cX\to[0,1]} \cL(h, \overline{\blambda}^T) 
    \leq \frac{JB}{\sqrt{T}} + \gamma.
\end{equation*}
Letting $(\vect h^*, \blambda^*)$ be primal dual optimal, we have 
\begin{equation}
\label{eq:gen2}
   \forall \blambda\in[0,B]^{K},\quad \cL(\vect h^*, \blambda^*) \geq \hat\cL(\bar{\vect h}^T, \blambda) - \frac{JB}{\sqrt{T}} - \gamma.
\end{equation}
The choices $\blambda=0$ and $\blambda=\blambda^* + \frac{B}{2}e_{g_m, \cir}$ give
\begin{align*}
    &\hat\cU(\bar{\vect h}^T) 
    \leq \cU(\vect h^*) + \gamma + \frac{JB}{\sqrt{T}}\\
    &\hat\cV(\bar{\vect h}^T)_k \leq \frac{2}{B}\left(\frac{JB}{\sqrt{T}} + 2\gamma\right).
\end{align*}
By Lemma~\ref{lem:gen}
\begin{equation*}
\forall g\in\gps,\quad \sup_{h\in\cH^{plg}}\|\mat C^g[h] - \hat{\mat C}^g[h]\|_{\infty} \leq 4\sqrt{\frac{K^2\log(2n_g)}{n_g}} + \sqrt{\frac{\log(2(|1+|\gps|)K^2/\delta)}{n_g}} =: \zeta(n_g).
\end{equation*}
we have that with probability $\geq 1-\delta$
\begin{align*}
    &\cU(\bar{\vect h}^T) \leq \cU(\vect h^*)+ \gamma + \frac{JB}{\sqrt{T}}+\rho\zeta(n_{\min}) \\
    &\cV(\bar{\vect h}^T)_k \leq \frac{2}{B}\left(\frac{JB}{\sqrt{T}} + 2\gamma\right)
    + \rho\zeta(n_{\min}).
\end{align*}
Therefore we obtain the bounds 
\begin{align*}
    &\cU_{\psi}(\bar{\vect h}^T) \leq \cU_{\psi}^*
    + \frac{JB}{\sqrt{T}}
    + ((1+J)B+1)\rho\left(4\sqrt{\frac{K^2\log(2n_{\min})}{n_{\min}}} + \sqrt{\frac{\log(2(1+|\gps|)K^2/\delta)}{n_{\min}}}\right) \\
    &\qquad\qquad\qquad\qquad + \E \|\beeta(x)-\hat\beeta(x)\|_1B\left(\rho_{\cX }+\sum_{g\in\gps}+ \frac{\rho_g }{\pi_g}\right) + 2\sqrt{\frac{\log(|\gps|/\delta)}{n}}\sum_{g\in\gps}\frac{\rho_g B}{\pi_g^2}\\
    &\|\cV_{\Phi}(\bar{\vect h}^T)\|_{\infty} \leq \frac{2J}{\sqrt{T}}
    +4(4(1+J)+1)\rho\left(\sqrt{\frac{K^2\log(2n_{\min})}{n_{\min}}} + \sqrt{\frac{\log(2(|1+|\gps|)K^2/\delta)}{n_g}}\right) \\
    &\qquad\qquad\qquad\qquad + 4\E \|\beeta(x)-\hat\beeta(x)\|_1\left(\rho_{\cX}+\sum_{g\in\gps}\frac{\rho_g}{\pi_g}\right) + 8\sqrt{\frac{\log(|\gps|/\delta)}{n}}\sum_{g\in\gps}\frac{\rho_g}{\pi_g^2}.
\end{align*}
\end{proof}

\section{Estimators}
In this section, we give plugin and weighted ERM methods of solving the linear probabilistic minimization problems arising from the Lagrangian of our fairness problem.  For clarity, we go over the choices of cost and constraint matrices corresponding to what we use in our experiments.

In our experiments, we maximize accuracy while enforcing independent demographic parity constraints and group-weighted gerrymandering demographic parity constraints. Under the framework of our probabilistic optimization problem, the former corresponds to the choice $\gps = \cG_\text{independent}$, and $\Phi$ containing the $2|\cG_\text{independent}| = 4M$ constraints 
\begin{equation*}
\forall\, g\in \cG_\text{independent},\,
\pm (\bC^g_{+,1}-\bC_{+,1}) \leq \nu,
\end{equation*}
where the $+$ subscript denotes summing over indices $0,1$ in place of $+$. I.e. for $g\in\cG_\text{indepdendent}$, $\bV_{g,\pm}^g = \pm\begin{bmatrix}0 & 1 \\ 0 & 1 \end{bmatrix}$,\, $\bV_{g,\pm}^{g'} = \boldsymbol{0}$ for $g\neq g'$, $\bU_{g,\pm} = \pm \begin{bmatrix} 0&1\\ 0 & 1 \end{bmatrix}$. $\bD = \begin{bmatrix} 0 & 1 \\ 1 & 0\end{bmatrix}$.\\ 
The latter corresponds to the choice $\gps=\cG_\text{gerrymandering}$, and the $2|\cG_\text{gerrymandering}|$ constraints 
\begin{equation*}
    \forall g\in \cG_\text{gerrymandering},\, 
    \pm \P(g)(\bC_{+,1}^g - \bC_{+,1})\leq \nu.
\end{equation*}
This corresponds to, for $g\in\cG_\text{gerrymandering},\,$ $\bV_{g,\pm}^g = \pm\P(g)\begin{bmatrix}0 & 1 \\ 0 & 1 \end{bmatrix}$,\, $\bV_{g,\pm}^{g'} = \boldsymbol{0}$ for $g\neq g'$, $\bU_{g,\pm} = \pm\P(g)\begin{bmatrix} 0&1\\ 0 & 1 \end{bmatrix}$.
The $\P(g)$'s will cancel out with the $\P(g)$'s in the expressions below.
\subsection{Plugin Estimator}
Using linearity of $\psi$ and $\phi$, if $\eta$ is known, the population minimizer $h^* = \argmin_{h:\cX\to[K]} \cL(h, \lambda)$ is deterministic and has a convenient closed form solution (the same is true of any linear minimization). 
\begin{align*}
\cL(h,\lambda)&=\langle\bD+\sum_{l=1}^L\lambda_l\bU_l, \C[h]\rangle-\sum_{g\in \gps}\sum_{l=1}^L \lambda_l\ip{\bV_l^g, \vec C^g[h]}
\\
&= 
\E\bigg\{\langle\bD+\sum_{l=1}^L\lambda_l\bU_l, \beeta(x)\vect h(x)^\top\rangle-\sum_{g\in \gps}\sum_{l=1}^L \lambda_l\big\langle\bV_l^g, \frac{\1_{\{x\in g\}}}{\P(g)}\beeta(x)\vect h(x)^\top\big\rangle\bigg\}\\
&=\E\beeta(x)^\top\big[\bD + \sum_{l=1}^L \lambda_l\bigg(\bU_l-\sum_{g\in\gps}\frac{\1_{x\in g}}{\P(g)}\bV_l^g\bigg)\big]\vect h(x).
\end{align*}
where we noticed that the conditional group confusion equals $\vec C^g[h] = \E\1_{\{x\in g\}}\beeta(x)\vect h(x)^\top/\P(g)$. Denote $\pi_g = \P(g)$ for $g\in\gps$ as the group probabilities. Thus, the minimizer has the deterministic form 
\begin{align}
    h^*(x) &= \argmin_{k\in\{1,\ldots, K\}} \bigg\{\eta(x)^\top\big[\bD + \sum_{l=1}^L \lambda_l\bigg(\bU_l-\sum_{g\in\gps}\frac{\1_{x\in g}}{\P(g)}\bV_l^g\bigg)\big]\bigg\}_k.
    \label{eq:plugsol}
\end{align}
Finally, since we do not actually have access to the true $\beeta$, we replace $\beeta$ with an estimated $\hat\beeta$. 

\subsection{Weighted ERM}
\label{section:werm-appendix}
In the weighed ERM approach (referred to as cost-sensitive classification for the binary case~\citep{agarwal18}) we parametrize $h:\cX \to[K]$ by a function class $\cF$ of functions $\bf:\cX\to\R^K$. The classification is the argmax of the predicted vector, $h(\bx) = \argmax_j(\bff(\bx)_j)$, so we denote the set of classifiers as $\cH^{werm} = \argmax\circ\cF$. For a standard classification problem with 0-1 error, minimizing the dataset error $\wh{\mathrm{err}}[h] = \frac{1}{n}\sum_{i=1}^n \1_{\{h(\bx_i)\neq y_i\}}$ is done by minimizing a surrogate loss $\ell:\R^K\times[K]\to\R_+$, e.g., using softmax cross-entropy, over the dataset, as $\hat\E\ell(\bff(\bx), y) = \frac{1}{n}\sum_{i=1}^n \ell(\bff(\bx_i), y_i)$. Then we take $h = \argmax\circ f$. 

Let $\ell(\bs)\in\R^k$ be the vector $\ell(\bs)_k = \ell(\bs,k)$.

In an analogous manner, we would like to minimize the empirical metric defined by the Lagrangian using a surrogate loss, as
\begin{align*}
    \min_{h\in\cH^{werm}} \hat\cL(h, \blambda)
    &= \sum_{i=1}^n e_{y_i}^\top\bigg[\frac{1}{n}\bD + \sum_{l=1}^L \frac{\lambda_l}{n}\bigg(\bU_l-\sum_{g\in\gps}\frac{\1_{x_i\in g}}{n_g}\bV_l^g\bigg)\bigg]\vect h(x_i).
\end{align*}
where $n_g = |\{i: x_i\in g\}|,\, g\in\gps$ are the empirical sizes of each group. Notice it has the form
\begin{equation*}
    \min_{h\in\cH^{werm}} \sum_{i=1}^n \bw_i^\top \vect h(x_i)
    = \sum_{i=1}^n s(\bw_i)\frac{\bw_i^\top}{s(\bw_i)} \vect h(x_i), \qquad s(\bw_i) = \frac{1}{n-1}\sum_{k=1}^K (\bw_i)_k.
\end{equation*}
If we interpret $\vec 1-\frac{\bw_i}{s(\bw_i)}$ as a probability distribution over labels and $s(\bw_i)$ as its weight, then we have $\min_h\tilde{\mathbb E}[(\vec 1-\tilde\beeta(x))^\top \vect h(x)]$ where $\tilde\P(x_i) = \frac{s(\bw_i)}{\sum_{i=1}^n s(\bw_i)}$ and $\tilde\beeta(x_i) =\vec 1- \frac{\bw_i}{s(\bw_i)}$. 

A priori, $\frac{\max_k (\bw_i)_k}{s(\bw_i)} \leq 1$, i.e. $\frac{\max_k (\bw_i)_k}{\sum_{k=1}^K (\bw_i)_k} \leq \frac{1}{n-1}$, may not hold. But, since shifting each entry of $w_i$ by the same amount does not change the initial optimization problem, we can add the constant amount $(n-1)\max_k (w_i)_k - \sum_{k=1}^K (w_i)_k$ to each entry of $w_i$, after which $\frac{\bw_i}{s(\bw_i)}\leq\vec 1$.

If $\ell$ is a surrogate loss used to minimize the multiclass error, it is assumed that we can minimize $\E[(1-\beeta(x))_{h(x)}]$ if we minimize $\E[\beeta(x)^\top \ell(f(x))]$ and take $h=\argmax\circ f$. Therefore, we can solve the weighted version by minimizing reweighted surrogate loss:
\begin{equation}
   \min_{f\in\cF} \tilde \E[\tilde\beeta(x)^\top\ell(f(x))] \equiv 
   \min_{f\in\cF} \sum_{i=1}^n s(\bw_i)\left(\vec1-\frac{\bw_i}{s(\bw_i)}\right)^\top \ell(f(x)) =: \hat L(f).
   \label{eq:rwl}
\end{equation}
This provides a convex surrogate for the original problem of minimizing the empirical Lagrangian.

\begin{lemma}[Confusion matrix generalization]
\label{lem:gen}
Denote $n_g$ as the number of samples belonging to group $g$ for $g\in\gps\cup\{\cX\}$.  Then with probability at least $1-\delta$, 
\begin{align*}
    \forall g\in \gps\cup\{\cX\},\; \sup_{h\in\conv\cH} \|\mat C^g[h] - \wh{\mat C}^g[h]\|_{\infty} \leq 4\sqrt{\frac{\vc(\cH)\log(n_g+1)}{n_g}} + \sqrt{\frac{\log((1+|\gps|)K^2/\delta)}{n_g}}.
\end{align*}
\end{lemma}
\begin{proof}
By standard binary classification generalization \citep{2005bouchsurvey}, with probability at least $1-\delta$,
\begin{align*}
    &\sup_{h\in\conv\cH}\left|P(Y=i, h(X)=j\mid g)-\hat P(Y=i, h(X)=j\mid g)\right|\\ &\qquad\qquad\qquad\qquad\qquad\leq 4\sqrt{\frac{\vc(\cH)\log(n_g+1)}{n_g}}
    + \sqrt{\frac{\log(1/\delta)}{n_g}}.
\end{align*}
Then we take a union bound over $|\gps|$ confusion matrices and $K^2$ entries per confusion matrix.
\end{proof}

\begin{theorem}
\label{thm:saddle}
Suppose $\psi : [0,1]^{K\times K}\to [0, 1]$ and $\Phi:[0,1]^{K\times K}\times ([0,1]^{K\times K})^{\gps}\to [0,1]^L$ are $\rho$-Lipschitz w.r.t. $\|\cdot\|_{\infty}$.
Recall $\hat\cL(\vect h,\blambda)=\hat\cE(\vect h)+\blambda^\top(\hat\cV(\vect h)-\ve\vec 1)$. Let $\gamma$ denote the bound in Lemma~\ref{lem:gen} that applies to $\C$,  $\gamma_g$ the bound that applies to $\C^g$, and denote $\gamma_{\gps} = \max_{g\in\gps} \gamma_g$.  If $\ve \geq \rho \gamma$ then with probability $1-\delta$:

If $(\bar{\vect h}, \bar\blambda)$ is a $\nu$-saddle point of $\max_{\blambda\in[0,B]^L}\min_{\vect h\in\conv\cH} \hat\cL(\vect h, \blambda)$, in the sense that $\max_{\blambda\in[0,B]^L} \hat\cL(\bar{\vect h}, \blambda)-\min_{\vect h\in\conv\cH}\hat\cL(\vect h, \bar\blambda)\leq\nu$, and $\vect h^*\in\conv\cH$ satisfies $\cV(\vect h^*)\leq 0$, then
\begin{align}
    & \cE(\bar{\vect h})\leq \cE(\vect h^*)+\nu+2\rho\gamma \label{eq:saddle1}\\
    & \|\cV(\bar{\vect h})\|_{\infty}\leq \frac{1+\nu}{B}+\rho\gamma_{\gps} +\ve. \label{eq:saddle2}
\end{align}
\end{theorem}
Thus, as long as we can find an arbitrarily good saddle point, which follows from weighted ERM if $\cH^{werm}$ is expressive enough while having finite VC dimension, then we obtain consistency.
\begin{proof}
    By Lemma~\ref{lem:gen}, with probability $1-\delta,$
    \begin{equation}
    \label{eq:saddle3}
    |\cE(\vect h)-\hat\cE(\vect h)| \leq \rho\gamma, \qquad  \|\cV(\vect h)-\hat\cV(\vect h)\|_{\infty}\leq \rho\gamma_{\gps}.
    \end{equation}
    Therefore, $\hat\cV(\vect h^*)\leq \ve$. Using this feasibility to argue the first inequality below:
\begin{align*}
\hat\cE(\bar{\vect h})-\hat\cE(\vect h^*)\leq \hat\cE(\bar{\vect h}) - 
\hat\cL(\vect h^*, \bar\blambda) = \hat\cL(\bar{\vect h}, 0)-\hat\cL(\vect h^*, \bar\blambda)\leq \nu.
\end{align*}
Then \eqref{eq:saddle1} follows from \eqref{eq:saddle3} and triangle inequality. For the next part,
\begin{equation*}
    B(\hat\cV(\bar{\vect h})_k-\ve) = \hat\cL(\bar{\vect h}, Be_k) - \hat\cL(\vect h^*,\bar\blambda) + \hat\cE(\vect h^*)-\hat\cE(\vect h)\leq \nu+1.
\end{equation*}
This and \eqref{eq:saddle3} imply \eqref{eq:saddle2}.
\end{proof}

\section{Datasets}
Here we dicsuss the datasets used and additional experimental details. 

\textbf{Communities and Crime:} contains neighborhoods featurized by various statistics pertaining to the neighborhoods, e.g. percent employed in various professions, demographics, rent, etc. The label is whether there is a high ($>70\%$-ile) rate of violent crimes per capita. There are $n=1994$ samples and $N=12$ protected attributes comprising various racial statistics.\\
\textbf{Adult census:} contains census data for $n=2020$ individuals. The label is whether an individual has high income. $N=7$ protected attributes comprising age, sex, and different races.\\
\textbf{German credit:} \citep{Dua2019} contains features such as financial holdings, occupation, housing, and reason for purchases, and the goal is to predict whether an individual has good credit. Several categorical variables were converted to one-hot encodings. There are $n=1000$ examples and $N=3$ protected attributes corresponding to age, sex, and foreign worker status.\\
\textbf{Law school:} contains $n=1823$ students and their gpas, cluster, and LSAT score. The goal is to predict whether the student passes the bar, and the protected attributes are age, gender, and family income.

For the constraint level $\nu$ we vary according a logarithmically spaced grid from $0.001$ to 1 with 20 points. We set $B=50$ for the \name\ methods. We vary the regularization parameter $\rho$ from $0.01/M$ to $1000/M$ across a logarithmically spaced grid with 20 points.

The authors of \citep{Kearns18} apply fictitious play to the gerrymandering problem, searching for the most violated constraint $\max_{g\in\gps} \frac{n_g}{n}|\bC^g_{0,1}+\bC^g_{1,1}-\bC_{0,1}-\bC_{1,1}|$ in response to the average of the predictors computed so far (if the violation exceeds $\nu$), and computing the minimizing predictor in response to the average of the dual variables obtained from the most violated constraints so far. On the other hand, we directly apply our \name\ framework to their original cost function (see\cite{Kearns18}) i.e., the problem of maximizing accuracy subject to $\forall g\in\gps,\, \frac{|g|}{n}|\bC^g_{0,1}+\bC^g_{1,1}-\bC_{0,1}-\bC_{1,1}|\leq \nu$. Both approaches aim to solve this problem.

Here are the full (training in addition to test) plots for the independent and gerrymandering experiments. 
\begin{figure}[th!]
\centering
    \includegraphics[width=\linewidth]{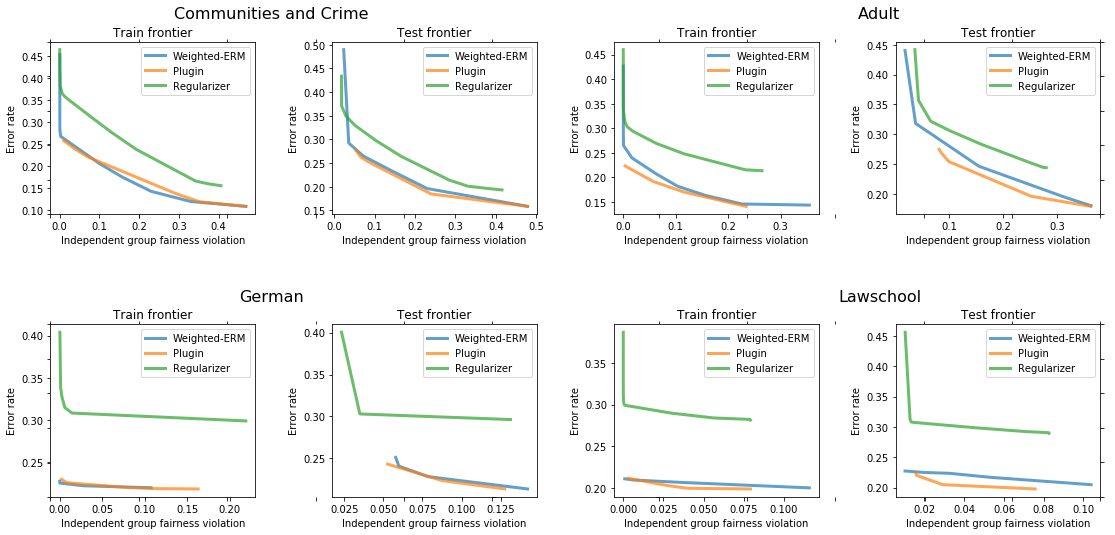}
    \caption{Experiments on independent group fairness. The pareto frontier closest to the bottom left represent the best fairness/performance tradeoff.}
    \label{fig:gerry_two}
\end{figure}

\begin{figure}[th!]
\centering
    \includegraphics[width=0.9\linewidth]{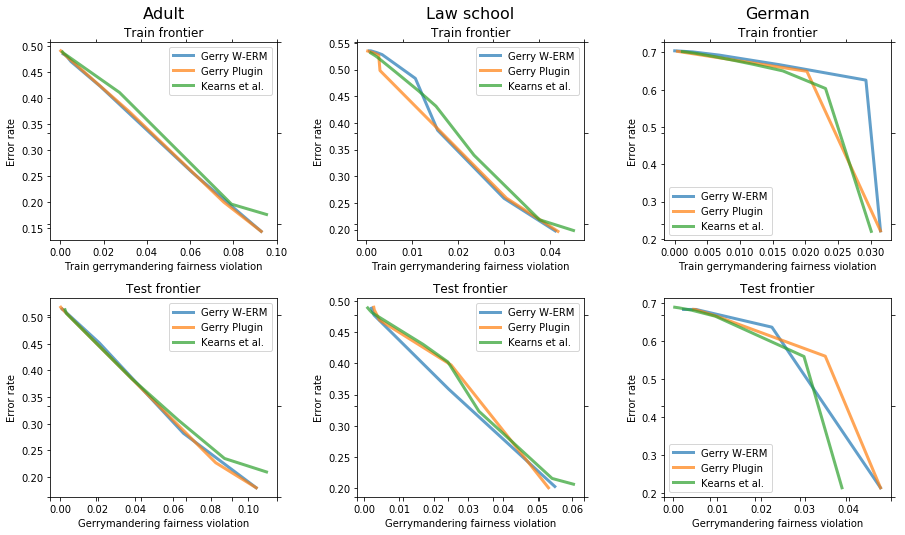}
    \caption{Experiments on gerrymandering group fairness. The pareto frontier closest to the bottom left represent the best fairness/performance tradeoff.}
    \label{fig:gerry_three}
\end{figure}

\end{document}